\documentclass[submission,copyright,creativecommons]{eptcs}
\providecommand{\event}{TARK-2021} % Name of the event you are submitting to
\usepackage{breakurl}             % Not needed if you use pdflatex only.
\usepackage{underscore}           % Only needed if you use pdflatex.
\usepackage{amssymb}
\usepackage{amsmath}
\usepackage{amsthm}
\usepackage{enumitem}
\usepackage{bm}
\usepackage[numbers]{natbib}
\usepackage{tikz}
\usepackage{natbib}

\tikzstyle{vertex}=[circle, draw, inner sep=0pt, minimum size=16pt]
\newcommand{\vertex}{\node[vertex]}
\def\hb{\hbox to 10.7 cm{}}

\usepackage[colorinlistoftodos]{todonotes}

\newcommand{\myendofproof}{~\hfill$\Box$}
\renewcommand{\geq}{\geqslant}

\renewcommand{\phi}{\varphi}

\newcommand{\prof}[1]{{\bm{#1}}}

\newcommand{\BAF}{\text{\it BAF}}   
\newcommand{\Arg}{\text{\it Arg}}          
\newcommand{\attacks}{\rightharpoonup}     
\newcommand{\supports}{\rightsquigarrow}

\renewcommand{\sup}{\text{\it sup}}
\newcommand{\Sup}{\text{\it Sup}}

\newcommand{\card}[1]{\lvert #1 \rvert} % cardinality of ...
\newcommand{\CL}[1]{\text{\rm CL}(#1)}

\newtheorem{definition}{Definition}
\newtheorem{theorem}{Theorem}
\newtheorem{lemma}[theorem]{Lemma}

\newtheorem{proposition}[theorem]{Proposition}
\newtheorem{fact}[theorem]{Fact}
\newtheorem{myexample}{Example}
\newenvironment{example}{\begin{myexample}\rm}{\hfill$\triangle$\end{myexample}}

\title{Collective Argumentation: The Case of Aggregating Support-Relations of Bipolar Argumentation Frameworks}
\author{Weiwei Chen
\institute{Institute of Logic and Cognition and Department of Philosophy\\
Sun Yat-sen University\\
Guangzhou, China}
\email{chenww26@mail2.sysu.edu.cn}
}
\def\titlerunning{Aggregation of Support-Relations of BAFs}
\def\authorrunning{Weiwei Chen}
\begin{document}
%\pagestyle{headings}
%\def\thepage{}

%%%%%%%%%%%%%%%%%%%%%%%%%%%%%%%%%%%%%%%%%%%%%%%%%%%%%%%%%%%%%%%%%%%%%%%%%%%%%%%%

%\begin{frontmatter}   

%\title{Collective Bipolar Argumentation: The Case of Aggregating Support-Relations}          
%\author{Weiwei Chen \\\small Institute of Logic and Cognition and Department of Philosophy \\\small Sun Yat-sen University, China\\\small Email: chenww26@mail2.sysu.edu.cn}
%\address{Institute of Logic and Cognition and Department of Philosophy \\ Sun Yat-sen University, China}
%\date{} 
\maketitle
\begin{abstract}
In many real-life situations that involve exchanges of arguments, individuals may differ on their assessment of which supports between the arguments are in fact justified, i.e., they put forward different support-relations. 
When confronted with such situations, we may wish to aggregate individuals' argumentation views on support-relations into a collective view, which is acceptable to the group.
In this paper, we assume that under bipolar argumentation frameworks, individuals are equipped with a set of arguments and a set of attacks between arguments, but with possibly different support-relations. 
Using the methodology in social choice theory, we analyze what semantic properties of bipolar argumentation frameworks can be preserved by aggregation rules during the aggregation of support-relations. %We obtain both possibility and impossibility results for semantic properties of bipolar argumentation frameworks.
\end{abstract}

%\begin{keyword}
%bipolar argumentation framework; social choice theory; collective argumentation
%\end{keyword}

%\end{frontmatter}

%%%%%%%%%%%%%%%%%%%%%%%%%%%%%%%%%%%%%%%%%%%%%%%%%%%%%%%%%%%%%%%%%%%%%%%%
%\pagebreak
\section{Introduction}\label{sec:introduction}
%Attack relations have been wildly considered in formal argumentation systems.
%Despite the highly abstract feature of these argumentation systems, 
%
The attack relation has played a significant role in formal argumentation~\cite{BesnardHunter2008,DungAIJ1995,RahwanSimari2009}.
However, recent years have seen a revived interest in the support relation between arguments in argumentation systems~\cite{BoellaEtAlCOMMA2010,CayrolLangasquieSchiexECSQARU2005A,CayrolLangasquieSchiexECSQARU2005B,CayrolLagasquieIJAR2013,CyrasEtAlPRIMA2017}.
In these systems, an argument can not only attack another argument, but it can also support another one.
For example, an argument can support another argument by confirming its premise or undermining one of its attackers.
The support relation between arguments is vital in modeling debates in real life. 
Due to the incompleteness of information, or different positions, agents may have different opinions regarding the support relation between arguments. To see this, consider the following example:
\begin{example}\label{example:ai}
Consider a debate regarding the possible influence of artificial intelligence (AI) to the job market.
Suppose that there are two arguments in this debate:\\

$A$: Artificial intelligence improves the degree of work automation

$B$: More people will lose their jobs due to AI \\

\noindent
Given the fact that AI is able to perform more of the tasks done by humans, some occupations will decrease. Therefore, some people hold that argument $A$ supports argument $B$. On the other hand, given that AI will improve the quality of the work being done by humans, lower the prices of goods and services, create economic advantages, and allow for the creation of new jobs in new occupations, some people hold that argument $A$ does not support argument $B$.
\end{example}
\pagebreak
In many scenarios, such as court debate, parliament debate, policy advisory committee decision-making, agents may have different opinions on which supports between arguments are acceptable, which form argumentative stances of them. When a group of agents are engaged in a debate, we may wish to aggregate stances possessed by agents to obtain a collective decision agreed on by the group.
To model the support relation between arguments, we consider the \emph{bipolar argumentation framework (BAF)}~\cite{CayrolLangasquieSchiexECSQARU2005A,CayrolLangasquieSchiexECSQARU2005B,CayrolLagasquieIJAR2013}, a formalism of Dung's abstract argumentation framework~\cite{DungAIJ1995}. 
%
%We also note from above examples that 
Given that there is a broad discussion of the aggregation of argumentation systems with the attack relation~\cite{CosteMarquisEtAlAIJ2007,TohmeEtAlFoIKS2008,DunneEtAlCOMMA2012,ChenEndrissAIJ2019}, it is far from being clear what consensuses can be achieved when the support relation is involved in this process.
The goal of this paper is to investigate the aggregation of views of a group of agents in the context of bipolar argumentation. %\footnote{I have previously outlined this idea in a preliminary version of this paper (\cite{ChenAAMAS2020}).}.
Given a set of arguments and a set of attack-relations between these arguments, agents might conflict with one another upon supports between arguments, i.e., for every pair of arguments that is being considered in a debate whether the first supports the second. In this scenario, we may wish to aggregate such support-relations. %\footnote{I have previously outlined this idea in a preliminary version of this paper (\cite{ChenAAMAS2020}).}.

In this paper, we use the method from \emph{graph aggregation}, a recent discipline of social choice theory that deals with aggregating several graphs into a single output graph that constitutes a good compromise. %The aggregation of specific types of graphs has been studied, among them, Chen and Endriss~\cite{ChenEndrissAIJ2019} study the problem of the aggregation of argumentation frameworks by making use of the methods from graph aggregation. 
Following the model introduced by Chen and Endriss~\cite{ChenEndrissAIJ2019}, we consider the preservation of properties of bipolar argumentation frameworks, %A property is a desirable feature of BAFs. 
i.e., given a property that is satisfied by individual BAFs, we study whether it can be satisfied in the BAF returned by aggregation rules. 
%We obtain both positive and negative results during the study. 
For some properties, we show that there is an aggregation rule or a family of aggregation rules preserve them. For some others, we show that any aggregation rule that satisfies certain basic axioms and preserves them must be a dictatorship.
%We demonstrate that there are several semantic properties %of bipolar argumentation frameworks 
%satisfying these two meta-properties, which implies that the preservation of such properties can only be fulfilled by dictatorships.
%
%Furthermore, we show how alternative instantiations of the general result from graph aggregation can generate new results in the domain of support relation aggregation.

\paragraph{Paper overview} The rest of the paper is organized as follows.
In Section~\ref{sec:baf}, we recapitulate the bipolar argumentation framework, along with its semantics.
We introduce our model for the aggregation of support-relations of bipolar argumentation frameworks in Section~\ref{sec:model}, followed by our results of preservation in Section~\ref{sec:results}.
%In Section~\ref{sec:properties}, we introduce properties of bipolar argumentation frameworks of particular interests
In Section~\ref{sec:related}, we introduce some work related to our work.
Finally, in Section~\ref{sec:conclusion}, we conclude this work and point out some directions for future work.

\section{Bipolar argumentation}\label{sec:baf}
An abstract bipolar argumentation framework~\cite{CayrolLangasquieSchiexECSQARU2005A,CayrolLangasquieSchiexECSQARU2005B,CayrolLagasquieIJAR2013} is an extension of Dung's abstract argumentation framework~\cite{DungAIJ1995} in which a general support relation between arguments is added.
Formally, an abstract bipolar argumentation framework is a triple $\langle \Arg, \attacks, \supports \rangle$, where $\Arg$ is a set of arguments, $\attacks$ is a binary relation on $\Arg$, which is called the attack relation, $\supports$ is a binary relation on $\Arg$, which is called the support relation. 
Given two arguments $A, B \in \Arg$, if $A \attacks B$ holds, then we say that $A$ attacks $B$, if $A \supports B$, then we say that $A$ supports $B$. 
The attack relation and the support relation must verify the following consistency constraint: $\attacks \cap \supports = \emptyset$, which is called \emph{essential constraint}.

\begin{definition}
Let $A, B \in \Arg$, there is a sequence of supports for $B$ by $A$ iff there exists a sequence of elements $(A_1,\ldots,A_n)$ of $\Arg$ such that $n \geq 2, A = A_1, B = A_n, A_1 \supports A_2, \ldots, A_{n-1} \supports A_n$.
\end{definition}

\begin{definition}
Let $A, B \in \Arg$, a \emph{supported attack} against $B$ by $A$ is a sequence of arguments $(A_1,\ldots,A_n)$  of $\Arg$ such that $A_1 \supports, \ldots, \supports A_{n-1}$, $A_{n-1} \attacks A_n$, $A = A_1$, $A_n = B$, and $n \geq 3$.
\end{definition}

%Note that if $A \attacks B$ is the case, then we say there is a direct attack from $A$ to $B$, or $A$ directly attacks $B$. % is considered as a supported attack with a sequence of two arguments $A$ and $B$.
Note that if $A \attacks B$ is the case, then we say that $A$ directly attacks $B$. % is considered as a supported attack with a sequence of two arguments $A$ and $B$.

\begin{definition}
A \emph{secondary attack} against an argument $B$ by an argument $A$ is a sequence $(A_1,\ldots,A_n)$ of arguments of $\Arg$ such that $A_1 \attacks A_2$, $A_2 \supports \ldots, \supports A_n$, $A = A_1$, $A_n = B$, and $n \geq 2$.
\end{definition}

For example, in Figure~\ref{fig:attack}, $A_1$ supported attacks $E_1$, while $A_2$ secondary attacks $E_2$.
\begin{figure}[h]
\[
\begin{tabular}{c@{\qquad}c@{\qquad}c@{\qquad}c}

\begin{tikzpicture}[->,>=latex,thick,shorten >=1pt,font=\small,scale=1]
  \vertex (A) at (0,0) {$A_1$};
  \vertex (B) at (1.5,0) {$B_1$};
  \vertex (C) at (3,0) {$C_1$};
  \vertex (D) at (4.5,0) {$D_1$};
  \vertex (E) at (6,0) {$E_1$};
  \draw [dashed, ->] (A) -- (B);
  \draw [dashed, ->] (B) -- (C);
  \draw [dashed] (C) -- (D);
  \draw [->] (D) -- (E);
\end{tikzpicture}\\Supported attack\\\\
\begin{tikzpicture}[->,>=latex,thick,shorten >=1pt,font=\small,scale=1]
  \vertex (A) at (0,0) {$A_2$};
  \vertex (B) at (1.5,0) {$B_2$};
  \vertex (C) at (3,0) {$C_2$};
  \vertex (D) at (4.5,0) {$D_2$};
  \vertex (E) at (6,0) {$E_2$};
  \draw [->] (A) -- (B);
  \draw [dashed, ->] (B) -- (C);
  \draw [dashed] (C) -- (D);
  \draw [dashed, ->] (D) -- (E);
\end{tikzpicture}\\Secondary attack\\
\end{tabular}
\]
\vspace*{-5pt}\caption{Illustration of supported attack and secondary attack\label{fig:attack}}
\end{figure}
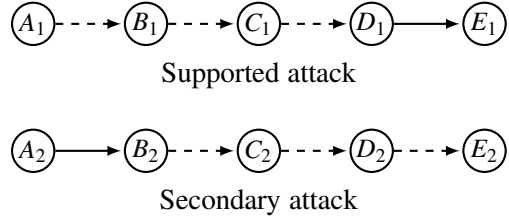

\begin{definition}
Let $\Delta \subseteq \Arg$ and $A \in \Arg$. $\Delta$ set-attacks $A$ iff there exists a supported attack or a secondary attack against $A$ from an element of $\Delta$. $\Delta$ set-supports $A$ iff there exists a sequence of supports for $A$ from an element of $\Delta$.
\end{definition}

\begin{definition}
Let $\Delta \subseteq \Arg$ be a set of arguments, $\Delta$ is conflict-free iff $\nexists A, B \in \Delta$ such that $\{A\}$ set-attacks $B$. %for no two argument $A, B \in \Arg$.
\end{definition}

\begin{definition}
Let $\Delta \subseteq \Arg$ be a set of arguments, $\Delta$ is \emph{safe} iff $\nexists B \in \Arg$ such that $\Delta$ set-attacks $B$ and either $\Delta$ set-supports $B$, or $B \in \Delta$.
\end{definition}

%Note that \emph{conflict-freeness} and \emph{safety} are two different lines of coherence\footnote{Another line of coherence combines conflict-freeness and closure for the support relation. But for the sake of simplicity, we will omit it in this paper}. 
%Recall that in the context of abstract argumentation, a set of arguments $\Delta \subseteq \Arg$ is admissible if it is being self-defending \footnote{A set of arguments $\Delta$ is called \emph{self-defending} if $\Delta \subseteq \{C\mid\Delta\ \text{defends}\ C\}$, i.e., $\Delta$ is a subset of the set of arguments defended by $\Delta$.} and conflict-free.
In the context of bipolar argumentation, admissibility can be translated into d-admissibility, \linebreak s-admissibility and c-admissibility, based on different lines of coherence.
%$\Delta$ is d-admissible if it is being self-defending and conflict-free (``d'' means ``in the sense of Dung''), $\Delta$ is s(afe)-admissible if it is it is being self-defending and safe, $\Delta$ is c(losed)-admissible if it is it is conflict-free, being self-defending and closed under support relations.
%
In the following definition, the notion of \emph{defense} is the same as classical defense, namely, we say $\Delta \subseteq \Arg$ defends the argument $B \in \Arg$, then, there is an argument $C \in \Delta$ with $C \attacks A$ for all arguments $A \in \Arg$ such that $A \attacks B$.
\begin{definition}
Let $\Delta \subseteq \Arg$ be a set of arguments, $\Delta$ is called \emph{d-admissible} iff $\Delta$ is conflict-free and defends all its elements; $\Delta$ is a d-preferred extension if it is maximal (w.r.t. set-inclusion) among all d-admissible sets.
\end{definition}

\begin{definition}
Let $\Delta \subseteq \Arg$ be a set of arguments, $\Delta$ is called \emph{s-admissible} iff $\Delta$ is safe and defends all its elements; $\Delta$ is a s-preferred extension if it is maximal (w.r.t. set-inclusion) among all s-admissible sets.
\end{definition}

Let the \emph{closure} of $\Delta \subseteq \Arg$ be $\CL{\Delta} = \{A \in \Arg\mid \text{there is a sequence of supports from }B \in \Delta\text{ to }A\}$, we say $\Delta$ is \emph{closed} iff $\Delta = \CL{\Delta}$.

\begin{definition}
Let $\Delta \subseteq \Arg$ be a set of arguments, $\Delta$ is called \emph{c-admissible} iff $\Delta$ is conflict-free, self-defending and closed; $\Delta$ is a c-preferred extension if it is maximal (w.r.t. set-inclusion) among all c-admissible sets.
\end{definition}

We restate a proposition in \cite{CayrolLangasquieSchiexECSQARU2005A} that demonstrates the relation between safety and conflict-freeness.
\begin{proposition}\label{prop:safe-closure}
Let $\Delta \subseteq \Arg$ be a set of arguments, if $\Delta$ is safe, then $\Delta$ is conflict-free. If $\Delta$ is conflict-free and closed, then $\Delta$ is safe.
\end{proposition}

\begin{definition}
Let $\Delta \subseteq \Arg$ be a set of arguments, $\Delta$ is stable if and only if $\Delta$ is conflict-free and for every argument $A \in \Arg \backslash \Delta$, $\Delta$ set-attacks $A$.
\end{definition}

It is worth mentioning that in the original papers, \cite{CayrolLangasquieSchiexECSQARU2005A,CayrolLangasquieSchiexECSQARU2005B} consider a particular set of BAFs, namly acyclic BAFs, showing that such BAFs have some nice features. 
However, in this paper, we foucs on BAFs that are more general, i.e., we remove the restriction on BAFs and consider both acyclic and cyclic BAFs.
From a technical point of view, the BAFs that are acyclic have only one stable extension, which is the only preferred extension as well, while the BAFs with cycles could have more than one stable extension and will be more general.

%\pagebreak
There are several interpretations of support in the literature, including the deductive support, the necessary support, and the evidential support (see an overview in \cite{CayrolLagasquieIJAR2013}).
The deductive support~\cite{BoellaEtAlCOMMA2010} is intended to capture the intuition that given two arguments $A$ and $B$, if $A$ supports $B$, then the acceptance of $A$ implies the acceptance of $B$.
The necessary support~\cite{NouiouaRischICTAI2010,NouiouaRischSUM2011} is intended to capture the intuition that if $A \supports B$ is the case, then the the acceptance of $B$ implies the acceptance of $A$, i.e., the acceptance of $A$ is necessary to obtain the acceptance of $B$.
Finally, the evidential support~\cite{OrenNorman2008semantics,OrenEtAlCOMMA2010} proposes a new type of argument, namely \emph{prima-facie} arguments. Every standard argument is supposed to be supported by at least one \emph{prima-facie} argument, and every \emph{prima-facie} argument does not require support from other arguments. %In this paper, we focus on the deductive support and will discuss necessary support, for the evidential support, we leave it to the future work.

The supported attack is connected with deductive support.
To see this, let us come back to Figure~\ref{fig:attack}, according to the deductive support, the acceptance of $A_1$ implies the acceptance of $B_1$, and so the acceptance of $C_1$, the acceptance of $D_1$. In the meantime, the acceptance of $D_1$ implies the non-acceptance of $E_1$. Thus, the acceptance of $A_1$ implies the non-acceptance of $E_1$.
The necessary support can be taken into account by considering secondary attack. We again consider Figure~\ref{fig:attack}. First, the acceptance of $A_2$ implies the non-acceptance of $B_2$. Then, according to necessary support, the non-acceptance of $B_2$ implies the non-acceptance of $C_2$, and so the non-acceptance of $D_2$, the non-acceptance of $E_2$. Thus, the acceptance of $A_2$ implies the non-acceptance of $E_2$.

\section{The model}\label{sec:model}
Fix a finite set $\Arg$ of arguments, a set $(\attacks)$ of attacks between arguments, and a set $N = \{1,\ldots,n\}$ of $n$ agents. Each agent $i \in N$ supplies us with a set of supports $\supports_i$, which together with $\Arg$ and $(\attacks)$ gives rise to a bipolar argumentation framework $\langle \Arg, \attacks, \supports_i \rangle$, reflecting her individual views on which supports between arguments are acceptable. %, we study aggregating these individual BAFs to a single BAF.
A \emph{profile} of support-relations $\prof {\supports} = (\supports_1,\ldots,\supports_n)$ is a set of support-relations provided by agents. 
An aggregation rule $F:(2^{\Arg \times \Arg})^n \to 2^{\Arg \times \Arg}$ is a function that maps a given profile of support-relations into a single support-relation.
We denote $N_{\sup}^{\prof {\supports}}$ by the set of agents who accept $\sup$ under profile $\prof {\supports}$, i.e., $N_{\sup}^{\prof {\supports}} = \{i \in N \mid \sup \in \supports_i\}$, and $\#N_{\sup}^{\prof {\supports}}$ denotes the number of such agents.

%
%In this context, the arguments and attack-relation are fixed, agents have different opinions on what support relation is acceptable.
Here we define desirable properties of aggregation rules.
These properties are referred as axioms in the social choice literature.
We start with formal definitions, followed by informal descriptions.

\begin{definition}
An aggregation rule $F$ is \emph{unanimous} if $\supports_1 \cap \ldots \cap\supports_n \subseteq F(\prof {\supports})$. 
\end{definition}

\begin{definition}
An aggregation rule $F$ is \emph{grounded} if $F(\prof {\supports}) \subseteq \supports_1 \cup \ldots \cup\supports_n$. 
\end{definition}

\begin{definition}
An aggregation rule $F$ is \emph{neutral} if for any profile of support-relations $\prof {\supports}$, for any pair of supports $\sup_1$, $\sup_2$, $N_{\sup_1}^{\prof {\supports}} = N_{\sup_2}^{\prof {\supports}}$ then $\sup_1 \in F(\prof {\supports})$ iff $\sup_2 \in F(\prof {\supports})$.
\end{definition}

\begin{definition}
An aggregation rule $F$ is \emph{independent} if for any pair of profiles of support-relations $\prof {\supports}_1$, $\prof {\supports}_2$, for any support $\sup$, $N_{\sup}^{\prof {\supports}_1} = N_{\sup}^{\prof {\supports}_2}$ then $\sup \in F(\prof {\supports}_1)$ iff $\sup \in F(\prof {\supports}_2)$.
\end{definition}

\begin{definition}
An aggregation rule $F$ is \emph{dictatorial} if there is an agent $i$ such that for any profile of support-relations $\prof {\supports}, F(\prof {\supports}) = \supports_i$.
\end{definition}

The \emph{unanimity} axiom states that the support agreed by all agents should be included in the collective BAF.
The \emph{groundedness} axiom expresses that all supports in the collective BAF should be supported by at least one agent.
The \emph{neutrality} axiom requires that given a profile of support-relations, any pair of supports should be treated equally in this profile.
The \emph{independent} axiom states that all support-relations should be treated equally in any profile of support-relations.
The \emph{dictatorship} axiom indicates that there is an agent who is dictatorial.
%
%\begin{definition}
%The (strict) majority rule is an aggregation rule $F$ with $F(\prof {\supports}) = \{\sup \in \Arg \times \Arg \mid\#N_{\sup}^{\prof {\supports}} > \frac{n}{2}\}$.
%\end{definition}
%other rule such as the nominatin rule, 
%
\begin{definition}
The unanimity rule is an aggregation rule $F$ with $F(\prof {\supports}) = \{\sup \in \Arg \times \Arg\mid \sup \in \supports_1 \cap \ldots \cap\supports_n\}$.
\end{definition}

\begin{definition}
Let $i \in N$ be an agent, the dictatorship rule of individual $i$ is the aggregation rule with $F(\prof {\supports}) = \supports_i$.
\end{definition}

%The majority rule returns a set containing only the majoritarian supports.
The unanimity rule only accepts those supports approved by all agents: it is a demanding aggregation rule.
%The dictatorships always return BAFs submitted by dictators.
The dictatorships always return the supports submitted by dictators.

\begin{example}
Suppose that there are three agents have to decide on the acceptance of supports between four arguments.
Agent 1 supports $A \supports B$ and $B \supports C$, agent 2 supports $B \supports C$ and $C \supports D$, agent 3 supports $A \supports B$ and $C \supports D$. 
We assume that the attack relation from $D$ to $E$ is accepted by all agents. 
The scenario is illustrated in Figure~\ref{fig:majority}. 
If we apply the majority rule, then we obtain a bipolar argumentation framework consisting of the three supports $A \supports B$, $B \supports C$, and $C \supports D$. 
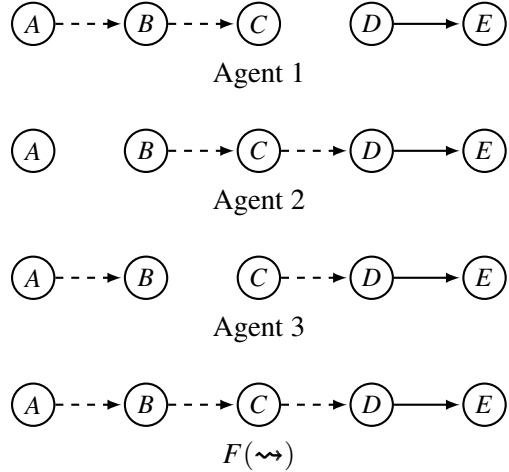
\begin{figure}
\[
\begin{tabular}{c@{\qquad}c@{\qquad}c@{\qquad}c}

\begin{tikzpicture}[->,>=latex,thick,shorten >=1pt,font=\small,scale=1]
  \vertex (A) at (0,0) {$A$};
  \vertex (B) at (1.5,0) {$B$};
  \vertex (C) at (3,0) {$C$};
  \vertex (D) at (4.5,0) {$D$};
  \vertex (E) at (6,0) {$E$};
  \draw [dashed, ->] (A) -- (B);
  \draw [dashed, ->] (B) -- (C);
  \draw [->] (D) -- (E);
\end{tikzpicture}\\Agent 1\\\\
\begin{tikzpicture}[->,>=latex,thick,shorten >=1pt,font=\small,scale=1]
  \vertex (A) at (0,0) {$A$};
  \vertex (B) at (1.5,0) {$B$};
  \vertex (C) at (3,0) {$C$};
  \vertex (D) at (4.5,0) {$D$};
  \vertex (E) at (6,0) {$E$};
  \draw [dashed, ->] (B) -- (C);
  \draw [dashed, ->] (C) -- (D);
  \draw [->] (D) -- (E);
\end{tikzpicture}\\Agent 2\\\\
\begin{tikzpicture}[->,>=latex,thick,shorten >=1pt,font=\small,scale=1]
  \vertex (A) at (0,0) {$A$};
  \vertex (B) at (1.5,0) {$B$};
  \vertex (C) at (3,0) {$C$};
  \vertex (D) at (4.5,0) {$D$};
  \vertex (E) at (6,0) {$E$};
  \draw [dashed, ->] (A) -- (B);
  \draw [dashed, ->] (C) -- (D);
  \draw [->] (D) -- (E);
\end{tikzpicture}\\Agent 3\\\\
\begin{tikzpicture}[->,>=latex,thick,shorten >=1pt,font=\small,scale=1]
  \vertex (A) at (0,0) {$A$};
  \vertex (B) at (1.5,0) {$B$};
  \vertex (C) at (3,0) {$C$};
  \vertex (D) at (4.5,0) {$D$};
  \vertex (E) at (6,0) {$E$};
  \draw [dashed, ->] (A) -- (B);
  \draw [dashed, ->] (B) -- (C);
  \draw [dashed, ->] (C) -- (D);
  \draw [->] (D) -- (E);
\end{tikzpicture}\\$F(\prof {\supports})$
\end{tabular}
\]
\vspace*{-5pt}\caption{Example for a profile with $\Arg=\{A,B,C,D,E\}$\label{fig:majority}}
\end{figure}
We observe that the set $\{A, E\}$ is conflict-free for all agents. However, it is not conflict-free in the outcome of the majority rule (which returns a set containing only the majoritarian supports) since $A$ supported attacks $E$. So conflict-freeness as a semantic property is not preserved by the majority rule in this specific example. 
\end{example}
But what about the preservation results of other semantic properties? Can they be preserved in general?
Before going any further, we introduce more semantic properties of particular interest.
%
%\section{Properties of BAFs}\label{sec:properties}

The problem we are considering in this paper is the preservation of semantic properties in the context of bipolar argumentation.
Given a property $P \subseteq 2^{\Arg \times \Arg}$ that is a set of supports on $\Arg$, and $P$ is satisfied by all agents, whether the output of the aggregation rule satisfies $P$? A formal definition is as follows.

\begin{definition}
An aggregation rule $F$ preserves a property $P$ if whenever for every profile $\prof {\supports}$
%and every property $P$ 
we have that $P(\supports_i)$ for all $i \in N$, then we have $P(F(\prof {\supports}))$.
\end{definition}

The problem of preservation is a special problem of \emph{collective rationality} which has been discussed extensively in other parts of social choice, such as preference aggregation~\cite{Arrow1963}, judgment aggregation~\cite{ListPettitEP2002}, graph aggregation~\cite{EndrissGrandiAIJ2017}, as well as attack aggregation in the context of abstract argumentation~\cite{BodanzaEtAlAC2017,RahwanTohmeAAMAS2010,ChenEndrissAIJ2019}.

%Given a bipolar argumentation framework, a property of the framework is a feature of it.
In the scenario where each agent possesses a BAF, agents might disagree on some details, such as whether a support between two arguments can be justified. 
Nevertheless, they may agree on some high-level features of BAFs.
The \emph{essential constraint} is an example of a high-level feature that requires no agent accepts both the attack relation and the support relation between a pair of arguments. When we observe that all agents verify such semantic feature, we would like to see what aggregation rule preserves this basic constraint under aggregation.

Given a set of arguments $\Delta \in \Arg$ that is conflict-free in every agent's bipolar argumentation framework, we may wish to preserve its conflict-freeness in the outcome. Therefore, conflict-freeness as a semantic property is of particular interest. Similar definition can be posed to the \emph{preservation of safety and admissibility}. Recall that if a set of arguments $\Delta$ is conflict-free and closed, then $\Delta$ is safe (Proposition~\ref{prop:safe-closure}). Thus, the \emph{closedness} is of interest to us as well.
Finally, we are also interested in the preservation of semantic extensions. Given a set of arguments $\Delta \subseteq \Arg$ that is an extension of a specific semantics of $\langle \Arg, \attacks, \supports_i \rangle$ for all $i \in N$, we are interested in under what circumstances $\Delta$ is an extension of such semantics of $F(\prof {\supports})$ as well. Finally, given an argument that is acceptable under a specific semantics for all agents, we would like to see whether such argument is acceptable in the collective outcome.

\section{Preservation results}\label{sec:results}
In this section, we present the preservation results for semantic properties. % defined in Section~\ref{sec:properties}.
We start with \emph{essential constraint} and \emph{closedness}, two basic requirements of bipolar argumentation frameworks.
Then, we turn to consider the preservation of \emph{conflict-freeness}, followed by considering \emph{safety}, followed by considering \emph{d-admissibility}, \emph{s-admissibility}, and \emph{c-admissibility}.
Then, we proceed with the study the properties of being an extension, including the property of \emph{being a d-preferred extension}, \emph{being a s-preferred extension}, \emph{being a c-preferred extension} and \emph{being a stable extension}. Finally, we study the preservation of \emph{acceptability of arguments}. Proofs of results in this section can be found in the appendix. 

\subsection{Preservation results for essential constraint, closedness, conflict-freeness, safety and admissibility}

Recall that a bipolar AF satisfies essential constraint if it does not contain two arguments for which the first one simultaneously attacks and supports the second one.
\begin{proposition}\label{prop:es}
Every aggregation rule $F$ that is grounded preserves essential constraint.
\end{proposition}

The closedness is also an important property. Our result demonstrates that every reasonable rule preserves it.
\begin{lemma}\label{lemma:cs}
Every aggregation rule $F$ that is grounded preserves closedness.
\end{lemma}

For conflict-freeness, we obtain that the unanimity rule, a demanding rule preserves the conflict-freeness of arbitrary sets of arguments.
\begin{proposition}\label{prop:cf}
The unanimity rule preserves conflict-freeness.
\end{proposition}

%We recall that the majority rule preserves the conflict-freeness during aggregation of attack relations.
%Recall that a set of arguments is safe if it is conflict-free and closed for the support relation. Similar to the preservation of conflict-freeness, 
The preservation of the safety of arbitrary sets of arguments can be accomplished by the unanimity rule.
\begin{proposition}\label{prop:sf}
The unanimity rule preserves safety.
\end{proposition}
\begin{proof}
This proposition is a consequence of Proposition~\ref{prop:cf}, Lemma~\ref{lemma:cs}, and Proposition~\ref{prop:safe-closure}.
\end{proof}

%We recall that the majority rule preserves the conflict-freeness during aggregation of attack relations.

The concepts of d-admissibility and s-admissibility are based on different coherences, but the preservation results for them are similar, as the following proposition demonstrates.
\begin{proposition}\label{prop:dadmissibility}
The unanimity rule preserves either d-admissibility or s-admissibility.
\end{proposition}
%
%
%We recall that the majority rule preserves the conflict-freeness during aggregation of attack relations.
%\begin{proposition}\label{prop:sadmissibility}
%The unanimity rule preserves s-admissibility.
%\end{proposition}
%We omit the relative easy proof of Proposition~\ref{prop:sadmissibility}. %Finally, we present the preservation result for c-admissibility.

\subsection{Preservation results for properties of being an extension}

We are going to present preservation results for more demanding properties.
Before proceeding, we introduce some necessary terminology and a simple result, as well as a technique developed by Endriss and Grandi for the more general framework of graph aggregation~\cite{EndrissGrandiAIJ2017}. 
%two meta-properties, namely non-simplicity and disjunctiveness.
%We first show that the preservation of a property that is \emph{non-simple} and \emph{disjunctive} implies that the aggregation rule is dictatorial.
Let $\sup \in \supports$ be a support, let $N = \{1,\ldots,n\}$ be a finite set of individuals (or agents, we assume that there are two or more agents), and let $\prof {\supports}$ be a profile of support-relations.
Recall that $N_{\sup}^{\prof {\supports}}$ is the set of agents who accept $\sup$ under profile $\prof {\supports}$.
A \emph{winning coalition} $\mathcal W \subseteq N$ is a set of agents who can decide whether to accept or reject a given support $\sup$.
Given an aggregation rule $F$, if $F$ is \emph{neutral} and \emph{independent}, then $F$ can be fully determined by a single set $\mathcal W$ of winning coalitions, i.e., for every profile $\prof {\supports}$ and every support $\sup$ it is the case that $\sup \in F(\prof {\supports}) \Leftrightarrow N_{\sup}^{\prof {\supports}} \in \mathcal W$.

In our proofs, we will rely on the concept of \emph{ultrafilter} familiar from set theory~\cite{MonjardetSCW1983}. An \emph{ultrafilter} is a collection of subsets of $N$ satisfying \emph{closure under intersection}, \emph{maximality}, and $\emptyset \notin \mathcal W$. 

\begin{definition}
An \emph{ultrafilter} $\mathcal W$ on a set $N$ is a collection of subsets of $N$ satisfying the following conditions:

\begin{enumerate}[label=$(\arabic*)$]
\item $\emptyset \notin \mathcal W$

\item for any pair of sets $C_1, C_2 \subseteq N$, $C_1, C_2 \in \mathcal W$ implies $C_1 \cap C_ 2 \in \mathcal W$ (closure under intersection)

\item for any set $C$, one of $C$ and $N\backslash C$ is in $\mathcal W$ (maximality)
\end{enumerate}
\end{definition}

\noindent
%Notably, every ultrafilter is a filter. In other words, every ultrafilter satisfies the condition of closure under supersets. 
We restate a simple result, which interprets a well-known fact of ultrafilter in our context:
\begin{quote}
\textit{Let $F$ be an independent and neutral aggregation rule and let $\mathcal W$ be the corresponding set of winning coalitions for supports, i.e., $\sup \in F(\prof {\supports}) \Leftrightarrow N_{\sup}^{\prof {\supports}} \in \mathcal W$ for all $\sup \in \supports$. Then, $F$ is dictatorial if and only if $\mathcal W$ is an ultrafilter.}
\end{quote}
\noindent
Besides the properties identified in Section~\ref{sec:model}, we introduce two meta-properties:
\begin{definition}
A property $P$ is called \textbf{non-simple} if there exists a set $\Sup\subseteq\Arg\times\Arg$ of supports and three individual supports $\sup_1,\sup_2,\sup_3\in\Arg\times\Arg\setminus\Sup$ such that $\langle\Arg,\attacks,\Sup\cup S\rangle$ with $S\subseteq\{\sup_1,\sup_2,\sup_3\}$ satisfies $P$ if and only if $S\not=\{\sup_1,\sup_2, \sup_3\}$. 
\end{definition}

\begin{definition}
A property $P$ is called \textbf{disjunctive} if there exists a set $\Sup\subseteq\Arg\times\Arg$ of supports and two individual supports $\sup_1,\sup_2\in\Arg\times\Arg\setminus\Sup$ such that $\langle\Arg,\attacks,\Sup\cup S\rangle$ with $S\subseteq\{\sup_1,\sup_2\}$ satisfies $P$ if and only if $S\not=\emptyset$. 
\end{definition}

Non-simplicity requires that, in the context of $\Sup$, accepting any proper subset of $\{sup_1, sup_2, sup_3\}$ is possible, while accepting  $\{sup_1, sup_2, sup_3\}$ is not.
Disjunctiveness requires that, in the context of $\Sup$, we should accept at least one of $sup_1$ and $sup_2$.
The term of \emph{simplicity} was introduced by Nehring and Puppe~\cite{NehringPuppeJET2007} as the \emph{median property}; the term of \emph{disjunctiveness} was introduced by Endriss and Grandi~\cite{EndrissGrandiAIJ2017} as a graph meta-property.
It is worth noting that a meta-property is a class of properties, a property satisfies or does not satisfy a specific meta-property.
Even though the definitions of meta-properties are not complicated, deciding whether a given property belongs to a meta-property is still not straightforward.
%For example, if we want to be certain whether the admissibility of a given set of arguments $\Delta \subseteq \Arg$ is conflict-free, we need to find out two supports $\sup_1, \sup_2 \in \supports$ such that in the BAFs that accepting one of $\sup_1$ and $\sup_2$, $\Delta$ is conflict-free. In the meantime, $\Delta$ is not conflict-free in the BAF in which both $\sup_1$ and $\sup_2$ are rejected.
%This is impossible, as removing a support from a conflict-free set will not make such set conflict. 

Meta-properties have connections with properties of BAFs and aggregation rules: on the one hand, meta-properties outline high-level features of properties of BAFs, with which we are able to systematically study the preservation of semantic properties of BAFs with simple proofs; on the other hand, as we will see in the following lemmas, meta-properties allow us to generalize specific result for specific properties by instantiating the general results. To be more specific, if an aggregation rule preserves a property that belongs to a meta-property, then it is a dictatorship. %The following lemma is an example.

\begin{lemma}\label{lemma:dictator}
Let $P$ be a property that is non-simple and disjunctive.
Then, for $\card{\Arg\,}\geq 3$, any unanimous, grounded, neutral, and independent aggregation rule~$F$ that preserves $P$ must be a dictatorship.
\end{lemma}

If a property we are interested in is non-simple and disjunctive, then we can apply Lemma~\ref{lemma:dictator} to obtain an axiomatic result for it.
\begin{theorem}\label{thm:s-preferred-dictator}
For $\card{\Arg\,}\geq 5$, any unanimous, grounded, neutral, and independent aggregation rule~$F$ that preserves either d-preferred, s-preferred, or c-preferred extensions must be a dictatorship.
\end{theorem} 

%We mention that in graph aggregation~\cite{EndrissGrandiAIJ2017} and attack relation aggregation~\cite{ChenEndrissAIJ2019}, a specific meta-property, namely implicativeness\footnote{In our context, a property is implicative if there are three supports $\sup_1$, $\sup_2$ and $\sup_3$,  such that accepting $\sup_1$ and $\sup_2$ implies accepting $\sup_3$} has been defined for preservation of certain properties under aggregation. 
%In this paper, this meta-property has been replaced by the meta-property of non-simplicity. But it is worth noting that they are more or less doing the same thing, namely served as a technical device to obtain impossibility results.
For the scenarios when $|\Arg| = 4$, or even $|\Arg| = 3$, we are not able to verify whether the theorem can still apply, we conjecture that the bound on the cardinality of $\Arg$ is sharp and we believe that the theorem has covered all cases of practical interest.
By comparison, we recall that for the property of being a preferred extension of Dung's argumentation framework, Chen and Endriss have shown that only dictatorships preserve it~\cite{ChenEndrissAIJ2019}. They have made the assumption that every agent is equipped with a different set of attack relations while they hold the same set of arguments.

Note that Theorem~\ref{thm:s-preferred-dictator} is an impossibility result that indicates the preservation of specific properties is impossible, unless the aggregation rule under consideration is dictatorial.
They relate to generalisations of Arrow's Impossibility Theorem~\cite{Arrow1963} to graph aggregation and attack-relation aggregation. One of the features of the aggregation rules we used in this section is that they accept the axiom of independence. Even though this axiom is attractive in some sense, to escape the impossibilities, a prime direction is to relax it. For example, we can consider distance-based rules and investigate whether we are able to obtain some positive results.

%\pagebreak
Theorem~\ref{thm:s-preferred-dictator} has assumed that the interpretation of support is deductive support. 
Even though it is enough for our purposes, namely it is enough to show that only dictatorships preserve d-preferred (s-preferred, c-preferred) extensions, we are still interested in what happens when the interpretation is restricted to necessary support. The bad news is, we still cannot overcome impossibility results. 
\begin{theorem}\label{thm:n-preferred-dictator}
If the interpretation of support is necessary support, for $\card{\Arg\,}\geq 5$, any unanimous, grounded, neutral, and independent aggregation rule~$F$ that preserves either d-preferred, s-preferred, or c-preferred extensions must be a dictatorship.
\end{theorem}

%\pagebreak
\subsubsection{Preservation result for stable extensions}

For stable extensions, by using the same techniques, we obtain a similar impossibility result.

\begin{theorem}\label{thm:stability-dictator}
For $\card{\Arg\,}\geq 5$, any unanimous, grounded, neutral, and independent aggregation rule~$F$ that preserves stable extensions must be a dictatorship.
\end{theorem} 
By comparison, we recall that the nomination rule preserves stable extensions of Dung's argumentation frameworks~\cite{ChenEndrissAIJ2019}.

%\pagebreak
\subsection{Preservation of argument acceptability}
Now, we move to study the preservation of acceptability of arguments.
Before proceeding further, it is thus important to keep in mind that our model has assumed that each agent $i \in N$ reports a set of supports $\supports_i$ on the same set of arguments and attack relations.
%(report (has, approve, assigned) the same set of attack relations)
Given two BAFs $\supports_1$ and $\supports_2$, if $\supports_1 \supseteq \supports_2$, namely the supports of $\supports_1$ is a superset of $\supports_2$, then a d-admissible (s-admissible, c-admissible, respectively) set of $\supports_1$ remains d-admissible (s-admissible, c-admissible, respectively) in $\supports_2$.

\begin{lemma}\label{lemma:d-accptability}
Given two BAFs $\supports_1$ and $\supports_2$, if $\supports_1 \supseteq \supports_2$, then every d-admissible set of arguments of $\supports_1$ is a d-admissible set of $\supports_2$.
\end{lemma}

\begin{lemma}\label{lemma:s-accptability}
Given two BAFs $\supports_1$ and $\supports_2$, if $\supports_1 \supseteq \supports_2$, then every s-admissible set of arguments of $\supports_1$ is a s-admissible set of $\supports_2$.
\end{lemma}

\begin{lemma}\label{lemma:c-accptability}
Given two BAFs $\supports_1$ and $\supports_2$, if $\supports_1 \supseteq \supports_2$, then every c-admissible set of arguments of $\supports_1$ is a c-admissible set of $\supports_2$.
\end{lemma}

\begin{fact}\label{fact:acceptability}
Given a BAF $\supports$, if $\Delta\subseteq\Arg$ is a d-preferred (s-preferred, c-preferred, respectively) extension of $\supports$, then $\Delta$ is a d-admissible (s-admissible, c-admissible, respectively) set of arguments of $\supports$.
\end{fact}

Thus, every d-admissible (s-admissible, c-admissible, respectively) set of arguments is included in a d-preferred (s-preferred, c-preferred, respectively) extension.
With this, we are ready to present the preservation results for argument acceptability under preferred semantics, including d-preferred semantics, s-preferred semantics, and c-preferred semantics.
Note that in the following we say that an argument under a d-preferred extension, we mean that such argument is a member of a d-preferred extension.

\begin{proposition}\label{proposition:dpf-acceptability}
The unanimity rule preserves the property of argument acceptability under d-preferred semantics.
\end{proposition}

\begin{proposition}\label{proposition:spf-acceptability}
The unanimity rule preserves the property of argument acceptability under either s-preferred or c-preferred semantics.
\end{proposition}
The proof is similar to the proof for Theorem~\ref{proposition:dpf-acceptability}. The only difference is that every s-admissible (c-admissible) set of $\supports_i$ is a s-admissible set of $F(\prof \supports)$ is because of Lemma~\ref{lemma:s-accptability} (Lemma~\ref{lemma:c-accptability}).

%\pagebreak
\section{Related work}\label{sec:related}
Previous work on obtaining argumentative consensus among a group of agents are mainly focus on abstract argumentation frameworks~\cite{CosteMarquisEtAlAIJ2007,TohmeEtAlFoIKS2008,ChenEndrissAIJ2019}.
Among them, Chen and Endriss~\cite{ChenEndrissAIJ2019} study of the preservation of semantic properties during the aggregation of attack-relations of abstract argumentation frameworks. As a potential domain of application for the model they develop, they do not make explicit reference to bipolar argumentation frameworks.
In addition, similar to us, they have made use of meta-properties proposed by Endriss and Grandi for graph aggregation~\cite{EndrissGrandiAIJ2017}, which serve as technical devices to obtain preservation results for semantic properties.

The problem of aggregating bipolar opinions has received interests in the literature.
The idea of aggregating support-relations of bipolar argumentation frameworks has been outlined in a preliminary version of this paper~\cite{ChenAAMAS2020}.
Lauren \textit{et al}.~\cite{LaurenEtAlAAMAS2021} consider aggregating bipolar assumption-based argumentation frameworks under the assumption that agents propose the same set of arguments, but propose different sets of attacks and supports.
Their focus is quota rules and oligarchic rules. %, while we take an axiomatic approach to the study of the preservation of semantic properties during the aggregation of support-relations over a common set of arguments and attacks.
%\cite{ChenEndrissAIJ2019} have studied the aggregation of semantic properties during aggregation of attack-relations of classical argumentation frameworks, while we study the preservation of semantic properties during aggregation of support-relations of BAFs.
%At the methodology level, Chen and Endriss have made use of three meta-properties from graph aggregation, namely the meta-properties of implicativeness, disjunctiveness, and contagiousness. We make use of disjunctiveness from graph aggregation and the concept of non-simplicity from judgment aggregation.
%
%Chen and Endriss~\cite{ChenEndrissAIJ2019} also consider the preservation of semantic properties during the aggregation of classical argumentation frameworks by making use of techniques from graph aggregation~\cite{EndrissGrandiAIJ2017}, and show that the preservation of certain properties is impossible.
%Different from Chen and Endriss, we have extended the application domain by proposing a new meta-property, namely non-simplicity. 
Kontarinis \textit{et al}.~\cite{KontarinisETAL2011} study designing mechanisms for ``regulating'' debates under the setting of each agent in the debate equipped with a bipolar argumentation framework.
We note that in their settings, agents report the same set of arguments, but with possibly different attack- and support-relations. 

\section{Conclusion}\label{sec:conclusion}

In this paper, we have studied the aggregation of agents' view in the context of bipolar argumentation.
To be more specific, we have explored the problem of aggregating support-relations of bipolar argumentation frameworks by making use of the methodology of social choice theory.
To achieve this, we have designed a model, in which agents are equipped with a set of arguments and a set of attacks, but with possibly different support-relations.
We have shown which semantic properties of BAFs can be preserved by aggregation rules.
We have proposed two BAF meta-properties, namely the property of ``non-simplicity'' and ``disjunctiveness'', both of which are high level features of BAFs.
We show that such meta-properties can be used to obtain impossibility results, namely for quickly proving what kind of aggregation rules (or no desirable aggregation rule) is collectively rational with respect to BAF properties.

For future work, it is worth having an investigation of further meta-properties.
%A property that is non-simple and disjunctive but not implicative, under graph aggregation (as well as attack aggregation), it is not clear whether there is an impossibility result for it. With Lemma~\ref{lemma:dictator}, we know that the preservation of such property is impossible. 
We point out that Lemma~\ref{lemma:dictator} is a variant of Theorem 16 in graph aggregation~\cite{EndrissGrandiAIJ2017}. This indicates that there is space for other variants of impossibility results for different assumptions.
%All these facts hint that an investigation to further meta-properties is potential.
The preservation of some desirable properties (for example, the property of being a d-preferred extension) during the aggregation of support-relations is difficult. 
Thus, it is worth studying whether such properties can be preserved in different settings. 
For instance, agents might be equipped with the same set of arguments, but with possibly different attack- and support-relations, and we aggregate attack- and support-relations by making use of different quota rules. 
Finally, recall that there are at least three interpretations of support, in this paper, we focus on the deductive support and necessary support. More interpretations of support, for example, evidential support, should be investigated. 

\paragraph{Acknowledgements.}
I would like to thank Ulle Endriss for his generous suggestions on an earlier version and his guidance at the beginning of this work. I also thank three anonymous reviewers for their constructive feedback. This work was supported by the China Postdoctoral Science Foundation (grant no. 2019M663352).

%\nocite{*}
\bibliographystyle{eptcs}
\bibliography{argumentation}

%%%%%%%%%%%%%%%%%%%%%%%%%%%%%%%%%%%%%%%%%%%%%%%%%%%%%%%%%%%%%%%%%%%%%%%%%%%%%%%%
\appendix
\section*{Appendix: Remaining Proofs}
%\gdef\thefigure{\arabic{figure}} % fixing bug in elsarticle.cls regarding figure numbering in appendix
%%%%%%%%%%%%%%%%%%%%%%%%%%%%%%%%%%%%%%%%%%%%%%%%%%%%%%%%%%%%%%%%%%%%%%%%%%%%%%%%

\noindent
In this appendix we present the proofs omitted from the body of the paper. 

\subsection*{Proof of Proposition~\ref{prop:es}}
Let $\prof {\supports} = (\supports_1, \ldots, \supports_n)$ be a profile of BAFs, in which $\supports_i$ satisfies the essential constraint for all $i \in N$.
Let $F$ be an aggregation rule that is grounded.
For the sake of contradiction, we suppose that the essential constraint is violated in $F(\prof {\supports})$.
Without loss of generality, we suppose that both $A \attacks B$ and $A \supports B$ get accepted in $F(\prof {\supports})$.
From the assumption, we know that every agent agrees $A \attacks B$.
In the meantime, at least one agent accepts $A \supports B$ under grounded aggregation rules, which cannot be the case.
Thus, we have the proposition.
\myendofproof

\subsection*{Proof of Lemma~\ref{lemma:cs}}
We again let $\prof {\supports} = (\supports_1, \ldots, \supports_n)$ be a profile of BAFs, let $\Delta \subseteq \Arg$ be the set of arguments under consideration, and let $F$ be an aggregation rule that is grounded.
For the sake of contradiction, we suppose that $\Delta$ is closed in $\supports_i$ for all $i \in N$ and not closed in $F(\prof {\supports})$, i.e., there is an argument $A \in \Delta$ and an argument $B \in \Arg \backslash \Delta$ such that $A \supports B$ in $F(\prof {\supports})$.
As rules under considering are grounded, there is at least one agent $\supports_i$ for which $A \supports_i B$ is the case. This will lead to the situation that $\Delta$ not being closed in $\supports_i$, contradicting our assumption.
\myendofproof

\subsection*{Proof of Proposition~\ref{prop:cf}}
Recall that the unanimity rule is the quota rule $F$ with the quota of $n$.
Let $\prof {\supports} = (\supports_1, \ldots, \supports_n)$ be a profile of BAFs.
Let $\Delta \subseteq \Arg$ be the set of arguments under consideration.
For the sake of contradiction, we suppose that $\Delta$ is conflict-free in $\supports_i$ for all $i \in N$, and is not conflict-free in $F(\prof {\supports})$. 
This means that there are two arguments $A, B \in \Delta$ such that $A$ supported or secondary attacks $B$ in $F(\prof {\supports})$.

We now show that in the scenario where $A$ is supported attacking $B$, i.e., there is a sequence of arguments in $F(\prof {\supports})$ such that $A_1 \supports, \ldots, \supports A_m$, $A_m \attacks B$ in which $A_1 = A$, our proposition holds. 
From the assumption we know that $(A_m \attacks B) \in \attacks_i$ for all $i \in N$. In addition, every agent agrees $A_1 \supports, \ldots, \supports A_m$.
Thus, every agent agrees $A_1 \supports, \ldots, \supports A_m$, $A_m \attacks B$, i.e., $\Delta$ is not conflict-free in $\supports_i$ for all $i \in N$, in contradiction to our earlier assumption.

For the scenario where $A$ is secondary attacking $B$, we note that the proof is similar to the proof of the one where $A$ is supported attacking $B$. Thus, we have the proposition.
\myendofproof

\subsection*{Proof of Proposition~\ref{prop:dadmissibility}}
Let $F$ be the unanimity rule.
Let $\prof {\supports} = (\supports_1, \ldots, \supports_n)$ be a profile of bipolar argumentation frameworks.
Let $\Delta \subseteq \Arg$ be the set of arguments under consideration.

We suppose that $\Delta$ is d-admissible in $\supports_i$ for all $i \in N$. Then, $\Delta$ is conflict-free in $\supports_i$ for all $i \in N$ as well. 
By Proposition~\ref{prop:cf}, $\Delta$ is conflict-free in $F(\prof {\supports})$.
By the assumption that all agents report the same set of attacks, we get that for each argument $A \in \Delta$, $A$ is defended by $\Delta$ in $\supports_i$ for all $i \in N$. It follows that $A$ is defended by $\Delta$ in $F(\prof {\supports})$ as well. Thus, $\Delta$ is d-admissible in $F(\prof {\supports})$.

We omit the relative easy proof for s-admissibility.
\myendofproof

\subsection*{Proof of Lemma~\ref{lemma:dictator}}
Take any property $P$ that is non-simple and disjunctive.
Take any aggregation rule $F$ that is unanimous, grounded, neutral, independent and preserves $P$. 
By the assumption that $F$ is neutral and independent, $F$ is determined by a set of winning coalitions $\mathcal W \subseteq 2^{N}$.
What we need to do is proving that $\mathcal W$ is an ultrafilter, i.e., to show that $\mathcal W$ is closed under intersection, $\mathcal W$ satisfies maximality, and $\emptyset \notin \mathcal W$.

\paragraph{$\emptyset \notin \mathcal W$} This is a direct consequence of the assumption that $F$ is grounded.

\paragraph{Maximality}
Take any set of agents $C \subseteq N$. 
Consider a profile in which exactly the individuals in $C$ propose $\sup_1$ and exactly those in $N \backslash C$ propose $\sup_2$.
Since $P$ is disjunctive, we know that one of $\sup_1$ and $\sup_2$ must be part of $F(\prof {\supports})$. Hence $C \in \mathcal W$ or $N\backslash C \in \mathcal W$.

\paragraph{Closure under intersection}
Take any two winning coalitions $C_1, C_2 \in \mathcal W$. Assume toward a contradiction that $C_1 \cap C_2 \notin \mathcal W$.
Consider a profile in which exactly the individuals in $C_1$ propose $\sup_1$, exactly the individuals in $C_2$ propose $\sup_2$, and exactly the individuals in $N \backslash (C_1 \cap C_2)$ propose $\sup_3$.
Now, since $C_1$ and $C_2$ are winning coalitions, $\sup_1$ and $\sup_2$ must be part of $F(\prof {\supports})$.
Hence, due to $C_1 \cap C_2 \notin \mathcal W$ and $\mathcal W$ satisfying maximality, we have $N \backslash (C_1 \cap C_2)\in \mathcal W$. Since the individuals in $N \backslash (C_1 \cap C_2)$ propose $\sup_3$, we have $\sup_3 \in F(\prof {\supports})$. 
But we have assumed that $F$ preserves non-simplicity, i.e., $\sup_1, \sup_2, \sup_3$ cannot be accepted together in $F(\prof {\supports})$. Thus, $C_1 \cap C_2 \in \mathcal W$.
\myendofproof

\subsection*{Proof of Theorem~\ref{thm:s-preferred-dictator}}
\noindent
Suppose $\card{\Arg\,}\geq 5$. Let $P$ be the properties representing a given set of arguments being either a d-preferred, a s-preferred or a c-preferred extension. 
%
%We recall the result from \cite{ChenEndrissAIJ2019} which shows that for $\card{\Arg\,}\geq 4$, any unanimous, grounded, and independent aggregation rule~$F$ that preserves $P$ must be neutral. It remains to 
%
Thus, by Lemma~\ref{lemma:dictator}, we need to show that $P$ is non-simple and disjunctive in this case.

\begin{figure}
\[
\begin{tabular}{c@{\qquad}c@{\qquad}c@{\qquad}c@{\qquad}c}

\begin{tikzpicture}[->,>=latex,thick,shorten >=1pt,font=\small,scale=1]
  \vertex (A) at (0,0) {$A$};
  \vertex (B) at (1.8,0) {$B$};
  \vertex (C) at (3.6,0) {$C$};
  \vertex (D) at (5.4,0) {$D$};
  \vertex (E) at (7.2,0) {$E$};
  \draw [dashed, ->] (A) --node[above,above]{\scriptsize$sup_1$} (B);
  \draw [dashed, ->] (B) --node[above,above]{\scriptsize$sup_2$} (C);
  \draw [dashed, ->] (C) --node[midway,above]{\scriptsize$sup_3$} (D);
  \draw [->] (D) edge[bend left=20] (E);
  \draw[->] (B) edge[loop] (B);
  \draw[->] (C) edge[loop] (C);
% \draw [->] (E) edge[bend left=40] (B);
% \draw [->] (E) edge[bend left=30] (C);
  \draw [->] (E) edge[bend left=20] (D);
\end{tikzpicture}&
\begin{tikzpicture}[->,>=latex,thick,shorten >=1pt,font=\small,scale=1]
  \vertex (A) at (0,0) {$A$};
  \vertex (B) at (2,0) {$B$};
  \vertex (C) at (3,1) {$C$};
  \vertex (D) at (3,-1) {$D$};
  \vertex (E) at (4,0) {$E$};
  \draw [dashed, ->] (A) --node[above,above]{\scriptsize$sup_1$} (C);
  \draw [dashed, ->] (A) --node[above,above]{\scriptsize$sup_2$} (D);
  \draw [->] (B) -- (C);
  \draw [->] (B) -- (D);
  \draw [->] (C) -- (E);
  \draw [->] (D) -- (E);
\end{tikzpicture}\\
Non-simplicity & Disjunctiveness
\end{tabular}
\]
\vspace*{-5pt}\caption{Scenarios used in the proof of Theorem~\ref{thm:s-preferred-dictator}\label{fig:s-prefer}}
\end{figure}
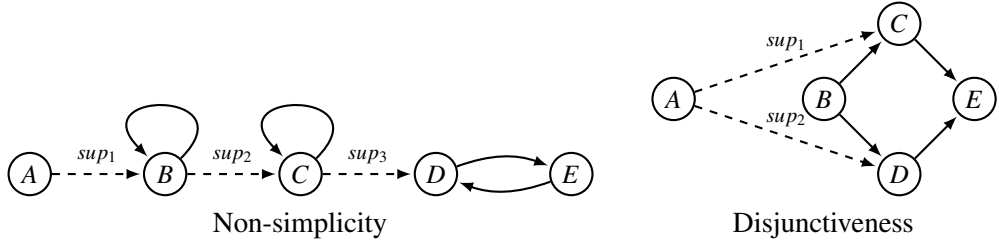

%The proofs of $P$ being disjunctive can be found in \cite{ChenEndrissAIJ2019} as well.

\paragraph{Non-simplicity}
% UE: checked with ConArg on 10/12/2017
Let $\Arg = \{A,B,C,D,E,\ldots\}$, let $\attacks = \{D \attacks E, E \attacks D, B \attacks B, C \attacks C\}$.
Now we focus on $\Arg\setminus\{B,C,D\}$ as the subset of arguments that may (or may not) form either a d-preferred, a s-preferred or a c-preferred extension. 
We define $\Sup = \emptyset$, $\sup_1 = (A \supports B)$, $\sup_2 = (B \supports C)$, and $\sup_3 = (C \supports D)$.
This scenario is depicted in the top part of Figure~\ref{fig:s-prefer}.
Consider all BAFs of the form $\BAF=\langle\Arg,\attacks,\Sup\cup S\rangle$ with $S\subseteq\{\sup_1,\sup_2,\sup_3\}$.
It is not difficult for the reader to verify that, for $S\not=\{\sup_1,\sup_2, \sup_3\}$, $B$ and $C$ are self-attacking, $D$ is attacked by $E$. Thus, they are unacceptable with respect to $\{A,E\}$, i.e., $\Arg\setminus\{B,C,D\}$ is a d-preferred, a s-preferred and a c-preferred extension. On the other hand, for $S=\{\sup_1,\sup_2,\sup_3\}$, $\Arg\setminus\{B,C,D\}$ is not a d-preferred nor a s-preferred or c-preferred extension as $E$ is set-attacked by $A$.
Thus, $P$ is non-simple.

\paragraph{Disjunctiveness}
Let $\Arg = \{A,B,C,D,E\ldots\}$, let $\attacks = \{B \attacks C, B \attacks D, C \attacks E, D \attacks E \}$.
We focus on $\Arg\setminus\{C,D,E\}$ as the subset of arguments that may (or may not) form a s-preferred extension. 
We define $\Sup = \emptyset$,
$\sup_1 = (A \supports C)$, $\sup_2 = (A \supports D)$.
This scenario is depicted in the bottom part of Figure~\ref{fig:s-prefer}.
Consider all BAFs of the form $\BAF=\langle\Arg,\attacks,\Sup\cup S\rangle$ with $S\subseteq\{\sup_1,\sup_2\}$.
For $S\not=\emptyset$, $\Arg\setminus\{C,D,E\}$ is a d-preferred, a s-preferred, and a c-preferred extension. On the other hand, for $S=\emptyset$, $\Arg\setminus\{C,D,E\}$ is not a s-preferred nor a d-preferred or a c-preferred extension as $E$ is defended by $B$. 
Thus, $P$ is disjunctive.
\myendofproof

\subsection*{Proof of Theorem~\ref{thm:n-preferred-dictator}}
Similar to Theorem~\ref{thm:s-preferred-dictator}, we still need to show that the property of being a d-preferred, a s-preferred, or a c-preferred extension is non-simple and disjunctive when the interpretation of support is necessary support.
\paragraph{Non-simplicity}
Let $\Arg = \{A,B,C,D,E,\ldots\}$, let $\attacks = \{D \attacks E, E \attacks D, B \attacks B, C \attacks C\}$.
Now we focus on $\Arg\setminus\{B,C,D\}$ as the subset of arguments that may (or may not) form either a d-preferred, a s-preferred or a c-preferred extension. 
We define $\Sup = \emptyset$, $\sup_1 = (B \supports A)$, $\sup_2 = (C \supports B)$, and $\sup_3 = (D \supports C)$, as illustrated in the top part of Figure~\ref{fig:n-prefer}.
Consider all BAFs of the form $\BAF=\langle\Arg,\attacks,\Sup\cup S\rangle$ with $S\subseteq\{\sup_1,\sup_2,\sup_3\}$.
It is not difficult for the reader to verify that, for $S\not=\{\sup_1,\sup_2, \sup_3\}$, $B$ and $C$ are self-attacking, $D$ is attacked by $E$. Thus, they are unacceptable with respect to $\{A,E\}$. In the meantime, $A$ is not attacked by any other argument, $E$ defends itself, $\{A,E\}$ is conflict-free, i.e., $\Arg\setminus\{B,C,D\}$ is a d-preferred, a s-preferred, and a c-preferred extension. On the other hand, for $S=\{\sup_1,\sup_2,\sup_3\}$, $\Arg\setminus\{B,C,D\}$ is neither a d-preferred nor a s-preferred nor c-preferred extension as $E$ secondary attacks $A$.
Thus, $P$ is non-simple.

\paragraph{Disjunctiveness}
Let $\Arg = \{A,B,C,D\ldots\}$, let $\attacks = \{B \attacks C, B \attacks D\}$.
We focus on $\Arg\setminus\{A,C,D\}$ as the subset of arguments that may (or may not) form a s-preferred extension. 
We define $\Sup = \emptyset$,
$\sup_1 = (C \supports A)$, $\sup_2 = (D \supports A)$, as illustrated in the bottom part of Figure~\ref{fig:n-prefer}.
Consider all BAFs of the form $\BAF=\langle\Arg,\attacks,\Sup\cup S\rangle$ with $S\subseteq\{\sup_1,\sup_2\}$.
For $S\not=\emptyset$, $B$ secondary attacks $A$ and directly attacks $C$ and $D$. Thus, $\Arg\setminus\{A,C,D\}$ is a d-preferred, a s-preferred, and a c-preferred extension. 
On the other hand, for $S=\emptyset$, $\Arg\setminus\{A,C,D\}$ is neither a s-preferred nor a d-preferred nor a c-preferred extension as $A$ is not attacked by any other argument and thus should be included in every d-preferred (s-preferred, c-preferred) extension. 
Thus, $P$ is disjunctive.
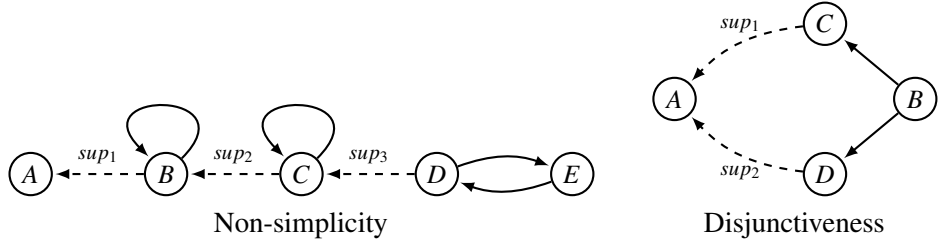
\begin{figure}
\[
\begin{tabular}{c@{\qquad}c@{\qquad}c@{\qquad}c@{\qquad}c}

\begin{tikzpicture}[->,>=latex,thick,shorten >=1pt,font=\small,scale=1]
  \vertex (A) at (0,0) {$A$};
  \vertex (B) at (1.8,0) {$B$};
  \vertex (C) at (3.6,0) {$C$};
  \vertex (D) at (5.4,0) {$D$};
  \vertex (E) at (7.2,0) {$E$};
  \draw [dashed, ->] (B) --node[above,above]{\scriptsize$sup_1$} (A);
  \draw [dashed, ->] (C) --node[above,above]{\scriptsize$sup_2$} (B);
  \draw [dashed, ->] (D) --node[midway,above]{\scriptsize$sup_3$} (C);
  \draw [->] (D) edge[bend left=20] (E);
  \draw[->] (B) edge[loop] (B);
  \draw[->] (C) edge[loop] (C);
% \draw [->] (E) edge[bend left=40] (B);
% \draw [->] (E) edge[bend left=30] (C);
  \draw [->] (E) edge[bend left=20] (D);
\end{tikzpicture}&
\begin{tikzpicture}[->,>=latex,thick,shorten >=1pt,font=\small,scale=1]
  \vertex (A) at (0,0) {$A$};
  \vertex (C) at (2,1) {$C$};
  \vertex (D) at (2,-1) {$D$};
  \vertex (B) at (3.2,0) {$B$};
  \draw [dashed, ->] (C) edge[bend right=20]node[above,above]{\scriptsize$sup_1$} (A);
  \draw [dashed, ->] (D) edge[bend left=20]node[above,below]{\scriptsize$sup_2$} (A);
  \draw [->] (B) -- (C);
  \draw [->] (B) -- (D);
\end{tikzpicture}\\
Non-simplicity & Disjunctiveness
\end{tabular}
\]
\vspace*{-5pt}\caption{Scenarios used in Theorem~\ref{thm:n-preferred-dictator}.\label{fig:n-prefer}}
\end{figure}
\myendofproof

\subsection*{Proof of Theorem~\ref{thm:stability-dictator}}
\noindent
Suppose $\card{\Arg\,}\geq 5$. Let $P$ be the BAF-properties representing a given set of arguments being a stable extension. 
%
%We recall the result from \cite{ChenEndrissAIJ2019} which shows that for $\card{\Arg\,}\geq 4$, any unanimous, grounded, and independent aggregation rule~$F$ that preserves $P$ must be neutral. It remains to 
%
We need to demonstrate that $P$ is non-simple and disjunctive in this case.
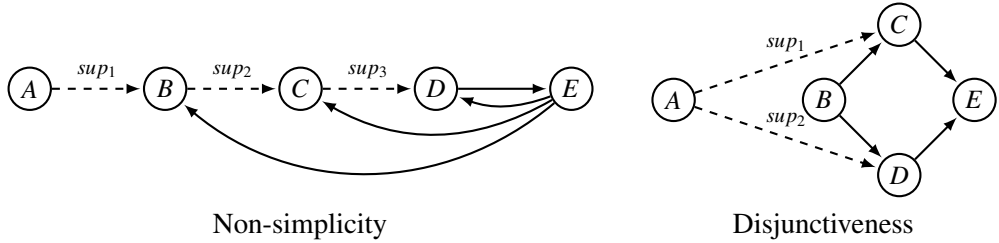
\begin{figure}[h]
\[
\begin{tabular}{c@{\qquad}c@{\qquad}c@{\qquad}c@{\qquad}c}

\begin{tikzpicture}[->,>=latex,thick,shorten >=1pt,font=\small,scale=1]
  \vertex (A) at (0,0) {$A$};
  \vertex (B) at (1.8,0) {$B$};
  \vertex (C) at (3.6,0) {$C$};
  \vertex (D) at (5.4,0) {$D$};
  \vertex (E) at (7.2,0) {$E$};
  \draw [dashed, ->] (A) --node[above,above]{\scriptsize$sup_1$} (B);
  \draw [dashed, ->] (B) --node[above,above]{\scriptsize$sup_2$} (C);
  \draw [dashed, ->] (C) --node[midway,above]{\scriptsize$sup_3$} (D);
  \draw [->] (D) -- (E);
  \draw [->] (E) edge[bend left=40] (B);
  \draw [->] (E) edge[bend left=30] (C);
  \draw [->] (E) edge[bend left=20] (D);
\end{tikzpicture}&
\begin{tikzpicture}[->,>=latex,thick,shorten >=1pt,font=\small,scale=1]
  \vertex (A) at (0,0) {$A$};
  \vertex (B) at (2,0) {$B$};
  \vertex (C) at (3,1) {$C$};
  \vertex (D) at (3,-1) {$D$};
  \vertex (E) at (4,0) {$E$};
  \draw [dashed, ->] (A) --node[above,above]{\scriptsize$sup_1$} (C);
  \draw [dashed, ->] (A) --node[above,above]{\scriptsize$sup_2$} (D);
  \draw [->] (B) -- (C);
  \draw [->] (B) -- (D);
  \draw [->] (C) -- (E);
  \draw [->] (D) -- (E);
\end{tikzpicture}\\
Non-simplicity & Disjunctiveness
\end{tabular}
\]
\vspace*{-5pt}\caption{Scenarios used in the proof of Theorem~\ref{thm:stability-dictator}\label{fig:stable}}
\end{figure}

%The proofs of $P$ being disjunctive can be found in \cite{ChenEndrissAIJ2019} as well.

\paragraph{Non-simplicity}
% UE: checked with ConArg on 10/12/2017
Let $\Arg = \{A,B,C,D,E,\ldots\}$, let $\attacks = \{D \attacks E, E \attacks B, E \attacks C, E \attacks D\}$.
We focus on $\Arg\setminus\{B,C,D\}$ as the subset of arguments that may (or may not) form a stable extension. 
We define $\Sup = \emptyset$.
$\sup_1 = (A \supports B)$, $\sup_2 = (B \supports C)$, and $\sup_3 = (C \supports D)$.
This scenario is depicted in the top part of Figure~\ref{fig:stable}.
Consider all BAFs of the form $\BAF=\langle\Arg,\attacks,\Sup\cup S\rangle$ with $S\subseteq\{\sup_1,\sup_2,\sup_3\}$.
The reader should be able to verify that, indeed, for $S\not=\{\sup_1,\sup_2, \sup_3\}$, $\Arg\setminus\{B,C,D\}$ is a stable extension. 
For example, for $S=\{\sup_1,\sup_2\}$, $B$, $C$, and $D$ are attacked by $E$. In the meantime, $\{A,E\}$ is conflict-free, i.e., $\Arg\setminus\{B,C,D\}$ is a stable extension. 
On the other hand, for $S=\{\sup_1,\sup_2,\sup_3\}$, $\Arg\setminus\{B,C,D\}$ is not a stable extension as $E$ is set-attacked by $A$.
Thus, $P$ is non-simple.

\paragraph{Disjunctiveness}
Let $\Arg = \{A,B,C,D,E\ldots\}$, let $\attacks = \{B \attacks C, B \attacks D, C \attacks E, D \attacks E \}$.
We focus on $\Arg\setminus\{C,D,E\}$ as the subset of arguments that may (or may not) form a stable extension. 
We define $\Sup = \emptyset$,
$\sup_1 = (A \supports C)$, $\sup_2 = (A \supports D)$.
This scenario is depicted in the bottom part of Figure~\ref{fig:stable}.
Consider all BAFs of the form $\BAF=\langle\Arg,\attacks,\Sup\cup S\rangle$ with $S\subseteq\{\sup_1,\sup_2\}$.
The reader should be able to verify that, indeed, for $S\not=\emptyset$, $\Arg\setminus\{C,D,E\}$ is a stable extension. On the other hand, for $S=\emptyset$, $\Arg\setminus\{C,D,E\}$ is not a stable extension as $E$ is defended by $B$. 
Thus, $P$ is disjunctive.
\myendofproof

\subsection*{Proof of Lemma~\ref{lemma:d-accptability}}
Let $\Delta \subseteq \Arg$ be a d-admissible set of $\supports_1$.
We need to show that $\Delta$ is a d-admissible set of $\supports_2$. To achieve this, we need to demonstrate that in $\supports_2$, ($i$) $\Delta$ is conflict-free, and ($ii$) $\Delta$ defends all of its members.

For ($i$), we need to show that $\Delta$ is conflict-free in $\supports_2$.
If not, then there are two arguments $A, B\in \Delta$ such $A$ directly, supported, or secondary attacks $B$.
If $A$ directly attacks $B$ in $\supports_2$, then $A$ directly attacks $B$ in $\supports_1$ as the pair of BAFs report the same set of attacks, which contradicts the assumption that $\Delta$ is conflict-free in $\supports_1$.
If $A$ supported attacks $B$ in $\supports_2$, then there is a sequence of argument $(A_1, \ldots, A_n)$ such that $A_1 \supports A_2, \ldots, A_{n-1} \attacks A_n$, $A= A_1$, and $A_n = B$. As $\supports_1 \supseteq \supports_2$ and $A_{n-1} \attacks A_n$ is the case, we know that $A_1 \supports A_2, \ldots, A_{n-1} \attacks A_n$ in $\supports_2$ as well, which means that there are two arguments $A, B\in\Delta$ such that $A$ supported attacks $B$, contradicting the fact that $\Delta$ is conflict-free in $\supports_1$.
If $A$ supported attacks $B$ in $\supports_2$, this case is similar to the case that $A$ supported attacks $B$, which will lead to that $\Delta$ failing to satisfy conflict-freeness in $\supports_1$.
Thus,  $\Delta$ is conflict-free in $\supports_2$.

For ($ii$), we need to show that for every argument $A \in\Delta$, if $B \attacks A$, then there is a $C \in \Delta$ such that $C \attacks B$, i.e, $\Delta$ defends all its members in $\supports_2$. Clearly, this is true as $\supports_1$ and $\supports_2$ report the same set of attacks, and $\Delta$ defends all its members in $\supports_1$.
\myendofproof

\subsection*{Proof of Lemma~\ref{lemma:s-accptability}}
We need to show that a s-admissible set of arguments $\Delta \subseteq \Arg$ of $\supports_1$ is a s-admissible set of $\supports_2$.
To arrive at this goal, we need to demonstrate that in $\supports_2$, ($i$) $\Delta$ is conflict-free, ($ii$) $\Delta$ defends all of its members, and ($iii$) $\Delta$ is safe.
For ($i$) and ($ii$), the proofs are the same as the ones in Lemma~\ref{lemma:d-accptability}. It remains to show that  $\Delta$ is safe in $\supports_2$.
If not, then there is an argument $B \in \Arg$ such that $\Delta$ set-attacks $B$ and $\Delta$ set-supports $B$, or $B \in \Delta$.
If $\Delta$ set-supports $B$.
there are two arguments $A\in \Delta$ such $A$ directly, supported, or secondary attacks $B$.
Using the construction similar to Lemma~\ref{lemma:d-accptability}, it is easy to verify that under this assumption, $\Delta$ set-attacks $B$ in $\supports_1$.
According to the assumption that $\supports_1 \supseteq \supports_2$, $\Delta$ set-supports $B$ in $\supports_1$.
Then, $\Delta$ is not safe in $\supports_1$, contradiction.
\myendofproof

\subsection*{Proof of Lemma~\ref{lemma:c-accptability}}
%Let $\Delta \subseteq \Arg$ be a c-admissible set of $\supports_1$.
Once again, we need to show that every c-admissible set $\Delta$ of arguments of $\supports_1$ is a c-admissible set of $\supports_2$.
To arrive at this goal, we need to show that in $\supports_2$ ($i$), $\Delta$ is conflict-free, ($ii$) $\Delta$ defends all of its members, and ($iii$) $\Delta$ is closed.
For ($i$) and ($ii$), the proofs are the same as the ones in Lemma~\ref{lemma:d-accptability}. It remains to show that  $\Delta$ is closed in $\supports_2$.
If not, then there is an argument $A \in \Delta$ and an argument $B \in \Arg$ such that $A$ supports $B$, and $B \notin \Delta$.
According to the assumption that $\supports_1 \supseteq \supports_2$, $A$ supports $B$, and $B \notin \Delta$ in $\supports_1$, we get that $\Delta$ is not closed in $\supports_1$, contradiction.
\myendofproof

\subsection*{Proof of Proposition~\ref{proposition:dpf-acceptability}}
Let $A\in\Arg$ be the argument under consideration, and we suppose that $A$ is acceptable under a d-preferred extension of $\supports_i$ for all $i \in N$.
Let $F$ be the unanimity rule.
Clearly, $F(\prof \supports)$ is a subset of $\supports_i$ for all $i\in N$.
Without loss of generality, we take $\supports_i$ to be the BAF under consideration.
Then, $\supports_i \supseteq F(\prof \supports)$. Furthermore, $A$ is acceptable under a d-preferred extension $\Delta_1 \subseteq \Arg$ of $\supports_i$ i.e., $A\in \Delta_1$.
Note that $\Delta_1$ is a d-admissible set as well.
According to Lemma~\ref{lemma:d-accptability}, $\Delta_1$ is a d-admissible set of $F(\prof \supports)$ as well. 
By Fact~\ref{fact:acceptability}, we know that there is a d-preferred extension $\Delta_2 \subseteq \Arg$ of $F(\prof \supports)$ such that $\Delta_2 \supseteq \Delta_1$, and $A$ is a member of $\Delta_2$. That is to say, $A$ is acceptable under a d-preferred extension of $F(\prof \supports)$. We are done.
\myendofproof

\end{document}